%% file: main.tex
\input{sections/00_preamble}

\begin{document}

\title{ManifoldKV: Training-Free KV Cache Compression via Euclidean Outlier Detection}

\author{
  Debajyoti Datta\thanks{Corresponding author: \texttt{debajyoti@hippocraticai.com}} \\
  Hippocratic AI
  \and
  Trishala Neeraj \\
  tn338@cornell.edu
  \and
  Bibek Paudel \\
  Hippocratic AI
  \and
  Vyom Sharma \\
  Hippocratic AI
  \and
  Subhabrata Mukherjee \\
  Hippocratic AI
}

\date{}

\maketitle

\input{sections/01_abstract}

\input{sections/02_introduction}

\input{sections/03_background}

\input{sections/04_method}

\input{sections/05_theory}

\input{sections/06_experiments}

\input{sections/07_discussion}

\input{sections/08_conclusion}

\bibliographystyle{plainnat}
\bibliography{example_paper}

\input{sections/09_appendix}

\end{document}

%% file: sections/00_preamble.tex
\documentclass[11pt]{article}

\usepackage[margin=1in]{geometry}

\usepackage{microtype}
\usepackage{graphicx}
\usepackage{subcaption}
\usepackage{booktabs}
\usepackage{hyperref}
\usepackage{url}

\usepackage{amsmath}
\usepackage{amssymb}
\usepackage{mathtools}
\usepackage{amsthm}
\usepackage{algorithm}
\usepackage{algorithmic}
\newcommand{\RETURN}{\STATE \textbf{return} }
\usepackage{xcolor}
\usepackage{multirow}
\usepackage{colortbl}
\usepackage{enumitem}  
\usepackage{natbib}

\usepackage[capitalize,noabbrev]{cleveref}

\usepackage{pifont}

\input{sections/00_figures_preamble}

\theoremstyle{plain}
\newtheorem{theorem}{Theorem}[section]
\newtheorem{proposition}[theorem]{Proposition}

\theoremstyle{definition}
\newtheorem{definition}[theorem]{Definition}
\newtheorem{assumption}[theorem]{Assumption}
\theoremstyle{remark}

\newcommand{\method}{\textsc{ManifoldKV}}
\newcommand{\windowmethod}{\textsc{WindowedManifoldKV}}
\newcommand{\R}{\mathbb{R}}
\newcommand{\E}{\mathbb{E}}
\newcommand{\bk}{\mathbf{k}}
\newcommand{\bv}{\mathbf{v}}
\newcommand{\bmu}{\boldsymbol{\mu}}
\newcommand{\cmark}{\textcolor{green!60!black}{\checkmark}}
\newcommand{\xmark}{\textcolor{red}{\ding{55}}}

%% file: sections/00_figures_preamble.tex
\usepackage{tikz}
\usepackage{pgfplots}
\pgfplotsset{compat=1.18}
\usepgfplotslibrary{groupplots}
\usetikzlibrary{arrows.meta,calc,shapes.geometric,positioning,fit,backgrounds}

\usepackage{float}

%% file: sections/01_abstract.tex
\begin{abstract}
Long-context inference is constrained by KV-cache memory, which grows linearly with sequence length; KV-cache compression therefore hinges on reliably selecting which past tokens to retain.
Most geometry-based eviction methods score keys by cosine similarity to a global centroid, but cosine is scale-invariant and can discard magnitude cues that distinguish semantically salient tokens.
We propose \method{}, a training-free scorer that ranks tokens by Euclidean distance to the key centroid, capturing both angular and radial deviations.

On the RULER benchmark, \method{} achieves \textbf{95.7\%} accuracy at 4K--16K contexts with 20\% compression; matching the best geometric baseline while improving robustness in two regimes where cosine scoring fails.
First, on multi-key retrieval, \method{} reduces directional collisions, achieving \textbf{92.4\%} vs KeyDiff's 77.0\% (+15.4 points) on 3-key NIAH at 50\% compression.
Second, to address dilution and performance collapse of global centroids at 64K context, we introduce \windowmethod{}, which restores accuracy to 84.3\% at 25\% compression, a 49-point recovery over global L2 and +3.2 points over KeyDiff. The method requires only 3 lines of code and works across 4 architectures without tuning.
\end{abstract}

%% file: sections/02_introduction.tex
\section{Introduction}
\label{sec:intro}

The Key-Value (KV) cache is an essential component in transformer based models, as it eliminates redundant computations by storing previously computed key and value vectors \citep{kwon2023efficient, li2024kvcachesurvey}.
While significantly speeding up text generation, it is also 
a critical memory bottleneck in long-context LLM inference. 
For a 70B model processing 100K tokens, the KV-cache alone requires $>$60 GB memory~\citep{kwon2023efficient}. 
Compression methods that evict less important tokens are essential for practical deployment, but every eviction strategy trades off between cache size and text generation accuracy: a small cache with a lot of tokens evicted is likely very inaccurate. 
Therefore, while designing an effective KV-cache compression method, one needs to address the question: \emph{which tokens to keep?}

\textbf{The Geometric Outlier Hypothesis.} Prior work on KV cache compression fall into two categories: \emph{attention-based methods} like SnapKV~\citep{li2024snapkv} and H2O~\citep{zhang2023h2o} that retain tokens receiving high cumulative attention scores, and \emph{geometric methods} like KeyDiff~\citep{park2025keydiff} that identify important tokens based on key vector geometry. Recent geometric approaches~\citep{knorm2024, park2025keydiff, feng2025identify} have converged on a compelling idea: tokens that are \emph{geometrically different} from the average context encode unique semantic content and should be retained. KeyDiff operationalizes this by measuring cosine similarity between each key vector $\bk_i$ and the mean $\bmu = \frac{1}{N}\sum_j \bk_j$, evicting tokens with high similarity (i.e., ``typical'' directions). 
This intuition is sound---critical entities like names, numbers, and technical terms should embed differently from common stopwords, hence they should not be evicted.

However, cosine similarity measures only \emph{angular} deviation, normalizing away vector magnitude entirely. Consider two key failure modes illustrated in Figure~\ref{fig:intuition}.
\textbf{(i) Radial outliers}: A token $\bk_i = \alpha \bmu$ with $\alpha \gg 1$ lies far from the centroid in Euclidean space, yet has \emph{maximum} cosine similarity (score = 1). Cosine-based methods incorrectly evict such tokens.
\textbf{(ii) Magnitude-encoded semantics}: Empirically, we find that semantically important tokens (entities, numbers) exhibit both unusual directions \emph{and} unusual magnitudes. The removal of magnitude discards important signal from the context.

For example, consider two tokens with keys $\bk_1 = 10\bmu$ and $\bk_2 = 0.1\bmu$ where $\bmu$ is the context centroid. 
Cosine similarity gives both identical scores: $\cos(\bk_1, \bmu) = \cos(\bk_2, \bmu) = 1$ (maximum similarity). Yet $\bk_1$ is 100$\times$ larger than $\bk_2$—geometrically very different. L2 distance correctly distinguishes them: $\|\bk_1 - \bmu\|_2 = 9\|\bmu\|_2$ vs $\|\bk_2 - \bmu\|_2 = 0.9\|\bmu\|_2$ (10$\times$ difference in score).


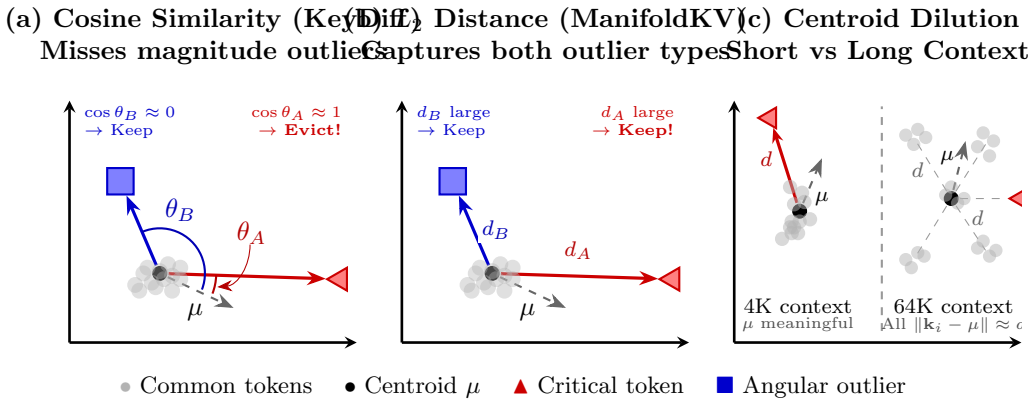
\begin{figure*}[ht]
    \centering
    \input{figures/figure_1.tex}
    \caption{\textbf{Geometric Intuition and the Centroid Dilution Problem.} \textbf{(a)} Cosine similarity (KeyDiff) captures only angular deviation---Token A (a radial outlier with $\bk_A = 2\bmu$) has $\cos \theta_A \approx 1$ and is incorrectly evicted. \textbf{(b)} L2 distance (ManifoldKV) captures both angular and radial deviation, correctly retaining both outliers. \textbf{(c)} The Centroid Dilution Problem: at short contexts (4K), tokens cluster around few themes and the centroid $\bmu$ is meaningful---outliers are clearly separable. At long contexts ($>$64K), tokens span many clusters; the centroid converges to a meaningless grand mean where all tokens appear equidistant.}
    \label{fig:intuition}
\end{figure*}

\textbf{Our Solution.} We propose \method{}, which scores tokens by their \emph{Euclidean} (L2) distance from the centroid:
\begin{equation}
    s_i = \|\bk_i - \bmu\|_2
\end{equation}
This simple change captures both angular and radial deviation. On the RULER benchmark~\citep{hsieh2024ruler}, \method{} achieves \textbf{95.7\%} accuracy at 4K--16K contexts across four architectures, outperforming attention-based SnapKV (84.0\%) by +11 points---without any model-specific tuning. Critically, ManifoldKV excels on \textbf{3-key needle-in-haystack (NIAH)} tasks: when multiple important tokens must be preserved, L2's magnitude awareness prevents \emph{directional collision}, outperforming KeyDiff by +15 points at aggressive compression.

\textbf{The long-context challenge.} While L2 distance excels at short-to-medium contexts (4K-16K), we discover a failure mode at $>$64K context-length: accuracy collapses from 82.3\% at 32K to just 35.2\% at 64K. 
We term this the \textbf{Centroid Dilution Problem}. When averaging over 64K diverse tokens spanning multiple topics and semantic domains, the centroid $\bmu$ becomes a meaningless ``center of mass'' that represents no coherent concept and cannot act as a discriminator of token importance. 
All tokens then appear approximately equidistant, destroying discriminative power.

As a solution to the Centroid Dilution Problem, we propose \windowmethod{}, which computes local centroids over sliding windows (e.g., 4K tokens). Each window maintains its semantic coherence, preserving the centroid's discriminative power. 
At 64K context, this recovers accuracy to \textbf{84.3\%}---a 49 percentage-point improvement over global L2 and +3.2 points over KeyDiff.

\textbf{Contributions.} Our three main contributions address distinct failure modes of existing methods:
\begin{enumerate}
    \item \textbf{Multi-Key Retrieval.} When multiple important tokens must be preserved, cosine-based methods suffer from \emph{directional collision}. 
    ManifoldKV's magnitude-aware L2 scoring outperforms KeyDiff by \textbf{+15.4 points} on 3-key NIAH at 50\% compression (92.4\% vs 77.0\%), demonstrating that magnitude information is critical for multi-needle retrieval.
    \item \textbf{Long-Context SOTA.} We identify the \emph{Centroid Dilution Problem}—global centroids become meaningless at 64K+ tokens, causing accuracy collapse (35.2\%). 
    WindowedManifoldKV with local sliding-window centroids recovers \textbf{+49 points} to 84.3\%, \textbf{+3.2 over KeyDiff}, achieving state-of-the-art on long-context compression.
    \item \textbf{Universal Manifold Structure:} Key vectors occupy a $\sim$9-dimensional manifold (Two-NN~\citep{facco2017estimating} estimate: 8.2--8.9) across all architectures, explaining zero-shot cross-model generalization (95--96\% with $\pm$0.3\% variance) and validating our $O(k)$ sample complexity advantage over cosine's fundamental failure on radial outliers.
    \item \textbf{Geometric vs Attention-Based Dominance:} ManifoldKV outperforms attention-based SnapKV by \textbf{+11 points} overall on RULER, demonstrating that geometric outlier detection is superior for retrieval tasks.
\end{enumerate}

\textbf{Paper Organization.} Section~\ref{sec:background} reviews KV cache compression and geometric scoring. Section~\ref{sec:method} presents ManifoldKV and WindowedManifoldKV. Section~\ref{sec:theory} provides theoretical analysis. Section~\ref{sec:experiments} validates our claims across 4 models and context lengths up to 64K.

%% file: figures/figure_1.tex

\begin{tikzpicture}

\begin{groupplot}[
  group style={group size=3 by 1, horizontal sep=6mm},
  width=5.4cm, height=4.8cm,
  xmin=-0.5,xmax=3.0, ymin=-0.3,ymax=1.8,
  axis lines=left, ticks=none,
  every axis/.append style={font=\small, line width=1.0pt},
  title style={font=\small\bfseries, align=center, at={(0.5,1.04)}, anchor=south},
]

\nextgroupplot[
  title={(a) Cosine Similarity (KeyDiff)\\Misses magnitude outliers},
]

\addplot[only marks, mark=*, mark size=2.5, gray!60, opacity=0.4] coordinates {
  (0.3,0.2) (0.4,0.3) (0.35,0.15) (0.45,0.35) (0.5,0.22)
  (0.55,0.18) (0.6,0.32) (0.65,0.25) (0.7,0.38) (0.75,0.28) (0.75,0.15)
  (0.85,0.25)
  (0.85,0.35)
};

\node[circle, fill=black, inner sep=2pt] (mu1) at (axis cs:0.6,0.3) {};
\node[anchor=south west, font=\small\bfseries] at (axis cs:0.8,-0.2) {$\mu$};

\draw[dashed, thick, black!60, -{Stealth[length=2.5mm]}]
  (mu1) -- (axis cs:1.5,0.0);

\node[rectangle, draw=blue!80!black, fill=blue!50, thick, minimum size=10pt, inner sep=0pt]
  (B1) at (axis cs:0.12,1.1) {};

\node[regular polygon, regular polygon sides=3, draw=red!80!black, fill=red!50,
      thick, minimum size=10pt, inner sep=0pt, shape border rotate=90]
  (A1) at (axis cs:2.8,0.25) {};

\draw[very thick, blue!80!black, -{Stealth[length=2.5mm]}] (mu1) -- (B1);
\draw[very thick, red!80!black, -{Stealth[length=2.5mm]}] (mu1) -- (A1);

\draw[blue!70!black, thick]
  let \p1=($(B1)-(mu1)$), \n1={atan2(\y1,\x1)},
      \p2=($(1.7,0.0)-(mu1)$), \n2={atan2(\y2,\x2)} in
  ($(mu1) + (\n1:0.6cm)$) arc[start angle=\n1, end angle=\n2, radius=0.6cm];
\node[blue!70!black, font=\small, fill=white, inner sep=1.5pt]
  at (axis cs:0.85,0.85) {$\theta_B$};

\draw[red!70!black, thick]
  let \p1=($(A1)-(mu1)$), \n1={atan2(\y1,\x1)},
      \p2=($(1.7,0.0)-(mu1)$), \n2={atan2(\y2,\x2)} in
  ($(mu1) + (\n1:0.75cm)$) arc[start angle=\n1, end angle=\n2, radius=0.75cm];

\node[red!70!black, font=\small, fill=white, inner sep=1.5pt]
  at (axis cs:1.72,0.65) {$\theta_A$};
\draw[red!70!black, -{Stealth[length=2mm, width=1.5mm]}, thin] (axis cs:1.70,0.55)
  to[out=-90, in=5] (axis cs:1.3,0.182);

\node[anchor=north west, font=\tiny, align=left, text=blue!80!black]
  at (axis cs:-0.45,1.85) {$\cos \theta_B \approx 0$ \\ $\rightarrow$ Keep};
\node[anchor=north east, font=\tiny, align=right, text=red!80!black]
  at (axis cs:2.95,1.85) {$\cos \theta_A \approx 1$\\ $\rightarrow$ \textbf{Evict!}};

\nextgroupplot[
  title={(b) $L_2$ Distance (ManifoldKV)\\Captures both outlier types},
]

\addplot[only marks, mark=*, mark size=2.5, gray!60, opacity=0.4] coordinates {
  (0.3,0.2) (0.4,0.3) (0.35,0.15) (0.45,0.35) (0.5,0.22)
  (0.55,0.18) (0.6,0.32) (0.65,0.25) (0.7,0.38) (0.75,0.28)
  (0.75,0.15)
  (0.85,0.25)
  (0.85,0.35)
};

\node[circle, fill=black, inner sep=2pt] (mu2) at (axis cs:0.6,0.3) {};
\node[anchor=south west, font=\small\bfseries] at (axis cs:0.8,-0.2) {$\mu$};

\draw[dashed, thick, black!60, -{Stealth[length=2.5mm]}]
  (mu2) -- (axis cs:1.5,0.0);

\node[rectangle, draw=blue!80!black, fill=blue!50, thick, minimum size=10pt, inner sep=0pt]
  (B2) at (axis cs:0.12,1.1) {};
\node[regular polygon, regular polygon sides=3, draw=red!80!black, fill=red!50,
      thick, minimum size=10pt, inner sep=0pt, shape border rotate=90]
  (A2) at (axis cs:2.8,0.25) {};

\draw[very thick, blue!80!black, -{Stealth[length=2.5mm]}] (mu2) -- (B2);
\draw[very thick, red!80!black, -{Stealth[length=2.5mm]}] (mu2) -- (A2);

\node[red!70!black, font=\scriptsize, fill=white, inner sep=1.5pt]
  at (axis cs:1.65,0.5) {$d_A$};
\node[blue!70!black, font=\scriptsize, fill=white, inner sep=1.5pt]
  at (axis cs:0.65,0.68) {$d_B$};

\node[anchor=north west, font=\tiny, align=left, text=blue!80!black]
  at (axis cs:-0.45,1.85) {$d_B$ large\\$\rightarrow$ Keep};
\node[anchor=north east, font=\tiny, align=right, text=red!80!black]
  at (axis cs:2.95,1.85) {$d_A$ large\\$\rightarrow$ \textbf{Keep!}};

\nextgroupplot[
  title={(c) Centroid Dilution\\Short vs Long Context},
  xmin=-0.3, xmax=3.3, ymin=-0.2, ymax=1.9,
]

\addplot[only marks, mark=*, mark size=2.2, gray!60, opacity=0.5] coordinates {
  (0.3,0.7) (0.45,0.8) (0.4,0.8) (0.5,0.75) (0.4,1.15)
  (0.55,1.1) (0.45,1.0) (0.58,1.02) (0.42,0.85) 
  (0.42,0.75) (0.52,0.85)  
  (0.65,0.87)
};


\node[circle, fill=black, inner sep=2pt] (mu_short) at (axis cs:0.52,0.94) {};
\node[anchor=south, font=\scriptsize\bfseries] at (axis cs:0.8,0.9) {$\mu$};

\node[regular polygon, regular polygon sides=3, draw=red!80!black, fill=red!50,
      thick, minimum size=9pt, inner sep=0pt, shape border rotate=90]
  (outlier_short) at (axis cs:0.15,1.75) {};

\draw[very thick, red!80!black, -{Stealth[length=2mm]}] (mu_short) -- (outlier_short);
\node[red!70!black, font=\scriptsize] at (axis cs:0.1,1.4) {$d$};

\draw[dashed, thick, black!60, -{Stealth[length=2.5mm]}]
  (mu_short) -- (axis cs:0.8,1.40);

\node[anchor=north, font=\scriptsize, align=center] at (axis cs:0.5,0.28) {4K context};
\node[anchor=north, font=\tiny, align=center, text=black!70] at (axis cs:0.5,0.12) {$\mu$ meaningful};

\draw[thick, black!40, dashed] (axis cs:1.55,-0.1) -- (axis cs:1.55,1.85);

\addplot[only marks, mark=*, mark size=2, gray!60, opacity=0.5] coordinates {
  (1.85,1.6) (2.0,1.65) (1.95,1.52) (2.1,1.58)
  (2.75,1.5) (2.9,1.42) (2.82,1.62) (2.95,1.55)
  (2.35,1.08) (2.5,1.01) (2.42,1.15) (2.58,1.05)
  (1.85,0.52) (2.0,0.58) (1.92,0.45) (2.08,0.55)
  (2.75,0.62) (2.9,0.55) (2.82,0.68) (2.98,0.6)
};

\node[circle, fill=black, inner sep=2pt] (mu_long) at (axis cs:2.42,1.05) {};
\node[anchor=south west, font=\scriptsize\bfseries] at (axis cs:2.15,1.25) {$\mu$};

\draw[thin, black!40, dashed] (mu_long) -- (axis cs:1.95,1.58);
\draw[thin, black!40, dashed] (mu_long) -- (axis cs:2.85,1.52);
\draw[thin, black!40, dashed] (mu_long) -- (axis cs:1.95,0.52);
\draw[thin, black!40, dashed] (mu_long) -- (axis cs:2.85,0.6);

\draw[dashed, thick, black!60, -{Stealth[length=2.5mm]}]
  (mu_long) -- (axis cs:2.6,1.55);

\node[regular polygon, regular polygon sides=3, draw=red!80!black, fill=red!50,
      thick, minimum size=9pt, inner sep=0pt, shape border rotate=90]
  (outlier_long) at (axis cs:3.3,1.05) {};
\draw[thin, black!40, dashed] (mu_long) -- (outlier_long);

\node[black!60, font=\scriptsize] at (axis cs:2.0,1.3) {$d$};
\node[black!60, font=\scriptsize] at (axis cs:2.75,0.88) {$d$};

\node[anchor=north, font=\scriptsize, align=center] at (axis cs:2.45,0.28) {64K context};
\node[anchor=north, font=\tiny, align=center, text=black!70] at (axis cs:2.45,0.12) {All $\|\bk_i - \mu\| \approx d$};

\end{groupplot}

\node[anchor=north, font=\footnotesize, align=center] at ($(current bounding box.south) + (0,-0.25cm)$) {
  \begin{tabular}{cccc}
    \textcolor{gray!60}{$\bullet$} Common tokens &
    \textcolor{black}{$\bullet$} Centroid $\mu$ &
    \textcolor{red!80!black}{$\blacktriangle$} Critical token &
    \textcolor{blue!80!black}{$\blacksquare$} Angular outlier
  \end{tabular}
};

\end{tikzpicture}

%% file: sections/03_background.tex
\section{Background and Related Work}
\label{sec:background}

We now introduce the KV-cache compression problem and review existing work in this domain.

\subsection{Setup and Notation.}
\label{sec:background_notation}

\textbf{Attention.}
Transformer models~\citep{vaswani2017attention} compute attention over a sequence of $N$ tokens. 
For each token position $i$, the model computes query ($Q$), key ($K$), and value ($V$) vectors. 
Attention at position $i$ is:
\begin{equation}
    \text{Attention}(q_i, K, V) = \sum_{j=1}^{N} \frac{\exp(q_i^\top k_j / \sqrt{d})}{\sum_{\ell=1}^{N} \exp(q_i^\top k_\ell / \sqrt{d})} v_j
\end{equation}
where $q_i$ is the query at position $i$, and $k_j, v_j$ are the key and value for position $j$.

\textbf{KV-Cache.}
During autoregressive generation, the model produces tokens sequentially: $t_1, t_2, \ldots, t_N$. 
At each step $t_i$, the model must attend to all previous tokens $\{t_1, \ldots, t_{i-1}\}$. 
Naively, this requires recomputing key and value vectors for all past tokens at every step.

The \emph{KV-cache} eliminates this redundancy by storing previously computed keys and values~\citep{pope2023efficiently, kwon2023efficient}, and works as follows.
After processing token $t_i$, cache its key $k_i$ and value $v_i$
At step $i+1$, load cached $\{k_1, \ldots, k_i\}$ and $\{v_1, \ldots, v_i\}$ from memory
and compute only the new key $k_{i+1}$ and value $v_{i+1}$.

While accelerating inference, the KV-cache creates a memory bottleneck. 
For a model with $L$ layers, $H$ attention heads, head dimension $d_h$, and context length $N$, the cache stores:
$\text{Cache size} = 2 \cdot L \cdot H \cdot N \cdot d_h \cdot$ bytes per element.
The factor of 2 accounts for both keys and values. 
Memory grows \emph{linearly} with context length $N$. 
For Llama-3.1-8B processing 64K tokens, the cache alone requires \textbf{8.6 GB}~\citep{kwon2023efficient} of memory. 
This severely limits long-context deployment, especially when serving multiple users concurrently.
Cache compression methods address this by evicting less important tokens, 
reducing $N$ to a smaller budget $M \ll N$ while maintaining text generation quality. 

\textbf{Token Eviction.}
We now formalize the token eviction process and explain how compressed caches integrate with standard attention mechanisms.

Consider a context of $N$ tokens $\mathcal{T} = \{t_1, \ldots, t_N\}$,
with each token $t_i$ having an associated key vector $\bk_i \in \mathbb{R}^d$ 
and value vector $\bv_i \in \mathbb{R}^d$ computed by the model, where $d$ is the model's hidden dimension. 
We organize these into matrices $K \in \mathbb{R}^{N \times d}$ and $V \in \mathbb{R}^{N \times d}$,
where row $i$ corresponds to token $t_i$.
Compression methods---including eviction~\citep{zhang2023h2o, li2024snapkv}, quantization~\citep{hooper2024kvquant}, and hybrid approaches~\citep{liu2024kivi}---assign importance scores $s_i$ to each token $i$ and evict low-scoring tokens. 
Given a compression ratio $\rho \in (0,1)$, the cache budget is $M = \lfloor (1-\rho) N \rfloor$. 
The eviction method must select which $M$ tokens to retain.
The key design choice is the scoring function.

A KV-cache compression method defines a scoring function $s: \mathbb{R}^d \rightarrow \mathbb{R}$ 
that assigns an importance score $s_i$ to each token $t_i$, given by
$s_i = s(\bk_i) \quad \text{for } i = 1, \ldots, N$

We then select the top-$M$ scoring tokens:
\begin{equation}
    \mathcal{I} = \text{TopK}(\{s_1, \ldots, s_N\}, M) \subseteq \{1, \ldots, N\}
\end{equation}
where $|\mathcal{I}| = M$ and $\mathcal{I}$ contains the indices of tokens to \emph{retain}.
We then keep only those selected tokens, creating a smaller KV-cache matrix $K'$ and $V'$.

After eviction, subsequent attention computations operate directly on $K'$ and $V'$:
$\text{Attention}(Q, K', V') = \text{softmax}\left(\frac{QK'^\top}{\sqrt{d}}\right) V'$,
where $Q \in \mathbb{R}^{N_q \times d}$ are query vectors for the next $N_q$ tokens to generate. 

\subsection{Related Work}
\label{sec:related}

\textbf{Benchmarks/Datasets.}

We evaluate on RULER~\citep{hsieh2024ruler}, a synthetic benchmark for long-context language models with 6,497 samples across context lengths from 4K to 128K tokens.
We focus on tasks that test retrieval and aggregation capabilities under compression.

\textbf{Needle-in-a-Haystack (NIAH).} NIAH tasks require retrieving key facts (``needles'') from long contexts (the ``haystack'').
Task variants include: single-key, multi-key, and multi-query retrieval.
Word extraction tasks include: common and frequent words extraction.
These tasks require attending to many tokens throughout the context, testing whether compression preserves global information.

Multi-key NIAH is particularly challenging, where performance drops from $>$95\% (single-key) to 77--92\% (3-key NIAH).
This task requires preserving multiple important tokens simultaneously, and
we find that existing methods suffer from \emph{directional collision},
where tokens in similar directions but different magnitudes are conflated (Section~\ref{sec:multikey_exp}).
Long contexts (64K) with aggressive compression (25\%) also show differences between methods.
We report average accuracy across RULER tasks, with detailed analysis of multi-key NIAH where ManifoldKV shows largest gains over baselines.

\textbf{Attention-Based Eviction.}
H2O \citep{zhang2023h2o} keeps tokens with highest cumulative attention (54.5\% accuracy on RULER at 64K). 
StreamingLLM \citep{xiao2024streamingllm} retains attention sinks plus a sliding window. SnapKV \citep{li2024snapkv} refines attention-based scoring with observation windows, achieving 83.9\% at 64K---the state-of-the-art for attention-based methods.

\textbf{Geometry-Based Eviction.}
Geometric methods score tokens based on key vector properties alone, avoiding attention computation. \textbf{KNorm}~\citep{knorm2024} uses magnitude $\|\bk_i\|_2$ alone (52.8\% on RULER). \textbf{KeyDiff}~\citep{park2025keydiff} uses cosine distance from the mean: $s_i = 1 - \cos(\bk_i, \bmu)$ where $\bmu = \frac{1}{N}\sum_i \bk_i$, achieving 81.1\% at 64K---a 28-point improvement over KNorm. \textbf{CriticalKV}~\citep{feng2025identify} extends geometric scoring to value vectors.
Cosine similarity \emph{discards magnitude}: $\cos(\alpha \bk_i, \bmu) = \cos(\bk_i, \bmu)$ for any $\alpha > 0$. 
This means radial outliers (tokens parallel to common directions but with extreme magnitudes) cannot be distinguished from typical tokens. Our work uses L2 distance $s_i = \|\bk_i - \bmu\|_2$, which captures both angular and radial deviation.

\textbf{Orthogonal Approaches.}
\textbf{AdaKV}~\citep{feng2024adakv} proposes adaptive per-head budget allocation---orthogonal to scoring. ManifoldKV integrates as a drop-in scorer within AdaKV. \textbf{DuoAttention}~\citep{xiao2024duoattention} specializes attention heads but requires calibration. Quantization methods~\citep{hooper2024kvquant, liu2024kivi} compress precision rather than evicting tokens.

\textbf{Manifold Structure.}
Our analysis builds on observations that neural representations lie on low-dimensional manifolds \citep{ansuini2019intrinsic,aghajanyan2021intrinsic}. We contribute the first application of manifold analysis to KV cache compression, showing key vectors occupy a universal $\sim$9D manifold.

%% file: sections/04_method.tex
\section{Method: ManifoldKV}
\label{sec:method}

We present \method{}, a simple yet effective scoring function for KV cache compression. We first introduce the core L2-based scoring (Section~\ref{sec:magnitude}), then analyze its failure mode at very long contexts (Section~\ref{sec:dilution}), and finally propose a windowed variant to address this limitation (Section~\ref{sec:windowed}).

\subsection{Magnitude-Aware Outlier Detection}
\label{sec:magnitude}

We propose that critical tokens differ from common tokens in \emph{two} geometric properties.
\textbf{(a) Angular deviation}: They point in unusual directions relative to the centroid (captured by cosine distance).
\textbf{(b) Radial deviation}: They have unusual magnitudes (ignored by cosine distance).

\begin{definition}[ManifoldKV Score]
Let $\bmu = \frac{1}{N}\sum_{i=1}^N \bk_i$ be the context centroid. The \method{} score for token $i$ is:
\begin{equation}
    s_i = \|\bk_i - \bmu\|_2
    \label{eq:manifoldkv}
\end{equation}
Tokens with high scores are geometric outliers and are retained during compression.
\end{definition}

\textbf{Geometric Decomposition.} To understand why L2 captures more information than cosine, we decompose the key vector as $\bk_i = r_i \hat{\bk}_i$, where $r_i = \|\bk_i\|_2$ is the magnitude and $\hat{\bk}_i = \bk_i / r_i$ is the unit direction vector. The squared L2 distance expands as:
\begin{align}
    \|\bk_i - \bmu\|_2^2 &= \|\bk_i\|^2 + \|\bmu\|^2 - 2\bk_i^\top \bmu \nonumber \\
    &= r_i^2 + \|\bmu\|^2 - 2r_i \|\bmu\| \cos(\bk_i, \bmu) \label{eq:decomposition}
\end{align}

This reveals three terms:
(a) $r_i^2$: The squared magnitude of the key vector,
(b) $\|\bmu\|^2$: A constant (same for all tokens), and
(c) $-2r_i \|\bmu\| \cos(\bk_i, \bmu)$: The angular alignment, scaled by magnitude.

In contrast, cosine similarity isolates only the angular component:
\begin{equation}
    \cos(\bk_i, \bmu) = \frac{\bk_i^\top \bmu}{\|\bk_i\| \|\bmu\|} = \frac{\hat{\bk}_i^\top \bmu}{\|\bmu\|}
\end{equation}
The magnitude $r_i$ cancels entirely. A token with $\bk_i = 10 \cdot \bmu$ (10$\times$ the centroid magnitude, same direction) receives identical cosine score to a token with $\bk_i = 0.1 \cdot \bmu$, yet these are geometrically very different. L2 distance correctly distinguishes them.

\begin{algorithm}[t]
\caption{\method{}: L2 Distance from Centroid}
\label{alg:manifoldkv}
\begin{algorithmic}
\STATE \textbf{Input:} Key tensor $K \in \R^{N \times d}$, compression ratio $\rho \in (0,1)$
\STATE \textbf{Output:} Indices $\mathcal{I}$ of tokens to retain
\STATE
\STATE $\bmu \leftarrow \frac{1}{N}\sum_{i=1}^N K_i$ \hfill \COMMENT{Compute centroid: $O(Nd)$}
\STATE $s_i \leftarrow \|K_i - \bmu\|_2 \;\; \forall i$ \hfill \COMMENT{L2 distances: $O(Nd)$}
\STATE $\mathcal{I} \leftarrow \text{TopK}(\mathbf{s}, \lfloor (1-\rho) N \rfloor)$ \hfill \COMMENT{Select top scores: $O(N \log N)$}
\RETURN $\mathcal{I}$
\end{algorithmic}
\end{algorithm}

\textbf{Complexity.} Algorithm~\ref{alg:manifoldkv} runs in $O(Nd + N \log N)$ time, dominated by the centroid computation and sorting. This is negligible compared to attention's $O(N^2 d)$ complexity. In practice, \method{} adds $<$0.5ms latency at 64K context (see Section~\ref{sec:experiments}).

\subsection{The Centroid Dilution Problem}
\label{sec:dilution}

\begin{table}[h]
\centering
\small
\caption{Average RULER accuracy (Llama-3.1-8B-Instruct) under standard compression settings: 20\% compression for 4K--32K and 25\% for 64K. Global (single-centroid) ManifoldKV collapses at 64K due to centroid dilution.}
\label{tab:dilution_intro}
\begin{tabular}{@{}lcccc@{}}
\toprule
\textbf{Context Length} & \textbf{4K} & \textbf{16K} & \textbf{32K} & \textbf{64K} \\
\midrule
\method{} Accuracy & 95.7\% & 92.8\% & 82.3\% & \textcolor{red}{35.2\%} \\
\bottomrule
\end{tabular}
\end{table}

While L2 distance excels at short-to-medium contexts (4K--32K tokens), we observe a \emph{catastrophic failure} at 64K (See Table~\ref{tab:dilution_intro} and Figure~\ref{fig:intuition}c).

\textbf{Diagnosis: Semantic Averaging.} At 64K tokens, the context typically spans multiple topics, entities, and semantic domains. The centroid $\bmu = \frac{1}{N}\sum_i \bk_i$ averages over this diverse set, converging to a ``center of mass'' that represents no coherent concept. When $\bmu$ is semantically meaningless, \emph{all} tokens appear approximately equidistant from it, and L2 scores lose discriminative power.


\textbf{Centroid Dilution (Informal).} When tokens span $K$ diverse semantic clusters, the global centroid converges to the grand mean of all clusters---a point representing nothing in particular. As $K$ grows, all tokens become approximately equidistant from this meaningless center, and L2 scores lose discriminative power. The sharp accuracy drop from 82.3\% (32K) to 35.2\% (64K) in Table~\ref{tab:dilution_intro} empirically confirms this prediction: beyond $\sim$32K tokens, global centroids become ineffective.

\subsection{Windowed Local Centroids}
\label{sec:windowed}

To combat centroid dilution, we compute \emph{local} centroids over sliding windows:

\begin{algorithm}[t]
\caption{\windowmethod{}: Local Centroids for Long Contexts}
\label{alg:windowed}
\begin{algorithmic}
\STATE \textbf{Input:} Keys $K \in \R^{N \times d}$, window size $W$, compression ratio $\rho$
\STATE \textbf{Output:} Indices $\mathcal{I}$ of tokens to retain
\STATE
\STATE $\mathbf{s} \leftarrow \mathbf{0}_N$ \hfill \COMMENT{Initialize scores}
\FOR{$t = 0, W, 2W, \ldots$ \textbf{while} $t < N$}
    \STATE $K_w \leftarrow K[t : \min(t+W, N)]$ \hfill \COMMENT{Extract window}
    \STATE $\bmu_w \leftarrow \text{mean}(K_w)$ \hfill \COMMENT{Local centroid}
    \STATE $\mathbf{s}[t:\min(t+W,N)] \leftarrow \|K_w - \bmu_w\|_2$ \hfill \COMMENT{Local L2 scores}
\ENDFOR
\STATE $\mathcal{I} \leftarrow \text{TopK}(\mathbf{s}, \lfloor (1-\rho) N \rfloor)$ \hfill \COMMENT{Global selection}
\RETURN $\mathcal{I}$
\end{algorithmic}
\end{algorithm}

\textbf{Why Windowing Works.} Each window of $W$ tokens (e.g., $W = 4096$) spans a limited semantic scope---typically a few paragraphs or a single topic. The local centroid $\bmu_w$ thus represents ``typical'' content \emph{within that region}, preserving discriminative power:

\begin{proposition}[Local Centroid Preservation]
\label{prop:local}
If tokens in window $[t, t+W)$ come from at most $K_w$ semantic clusters with $K_w \ll W / \sigma^2$, then the local centroid $\bmu_w$ remains within $O(\sigma)$ of the dominant cluster mean, and L2 scores retain discriminative power.
\end{proposition}

\textbf{Window Size Selection.} We find that $W = 4096$ works best, matching the context length where global \method{} achieves peak performance (95.7\%). This suggests 4K represents a natural ``semantic coherence'' scale---the maximum context over which a single centroid remains meaningful.

\textbf{Theoretical Justification.} Our manifold analysis (Section~\ref{sec:theory}) shows key vectors occupy a $k \approx 9$ dimensional space. Statistical theory suggests centroid estimation in $k$-D requires $O(k \log k) \sim 20$--50 samples for stability. While 4096 tokens vastly exceeds this minimum, empirically this window size achieves optimal accuracy (Table~\ref{tab:window_ablation}), suggesting 4K is the natural semantic coherence scale where a single centroid remains meaningful before topic drift occurs.

\begin{table}[h]
\centering
\small
\caption{Window size ablation at 64K context on RULER benchmark (Llama-3.1-8B, 25\% compression). 4K windows achieve optimal accuracy.}
\label{tab:window_ablation}
\begin{tabular}{@{}lccc@{}}
\toprule
\textbf{Window Size} & \textbf{64K Accuracy} & \textbf{$\Delta$ vs Global} & \textbf{$\Delta$ vs KeyDiff} \\
\midrule
Global (no window) & 35.2\% & --- & $-$45.9 \\
16K & 82.4\% & +47.2 & +1.3 \\
8K & 83.9\% & +48.7 & +2.8 \\
\rowcolor{green!10}
\textbf{4K} & \textbf{84.3\%} & \textbf{+49.1} & \textbf{+3.2} \\
2K & 83.8\% & +48.6 & +2.7 \\
\bottomrule
\end{tabular}
\end{table}

\textbf{Complexity.} \windowmethod{} has the same asymptotic complexity as the global variant: $O(Nd)$ for scoring plus $O(N \log N)$ for selection. The constant factor increases by $\lceil N/W \rceil$ centroids, but this is negligible in practice ($<$1ms overhead at 64K).

%% file: sections/05_theory.tex
\section{Theoretical Analysis}
\label{sec:theory}

We provide theoretical grounding for \method{}'s empirical success: L2 achieves $O(k)$ sample complexity where $k$ is the intrinsic dimension, while cosine fails on radial outliers regardless of sample size.

\subsection{Sample Complexity}
\label{sec:sample_complexity}

We assume common tokens concentrate near a $k$-dimensional subspace ($k \ll d$), with outliers at distance $\geq \epsilon$ from this subspace---validated empirically below ($k \approx 9$).

\textbf{L2 Sample Complexity.} Under this assumption, L2 distance identifies all outliers with $n = O(k \sigma^2 / \epsilon^2)$ samples, where $\sigma^2$ is the variance of common tokens. The key insight: concentration bounds on the $k$-dimensional subspace give centroid convergence rate $O(\sigma\sqrt{k/n})$, independent of the ambient dimension $d$.

\textbf{Cosine Failure.} Radial outliers $\bk_o = \alpha \bmu$ ($\alpha \gg 1$) have $\cos(\bk_o, \bmu) = 1$ (maximum similarity) despite being geometrically distant. Cosine is scale-invariant, so such outliers are undetectable regardless of sample size or algorithm sophistication. L2 correctly identifies them: $\|\bk_o - \bmu\| = (\alpha - 1)\|\bmu\| \gg 0$. Full proofs in Appendix~\ref{app:theory}.

\subsection{Universal Manifold Structure}
\label{sec:manifold_empirical}

We validate our low-dimensional assumption using the Two-NN estimator~\citep{facco2017estimating} (Table~\ref{tab:pca}):

\begin{table}[h]
\centering
\caption{\textbf{Manifold Dimension.} Key vectors occupy a universal $\sim$9D manifold regardless of architecture.}
\label{tab:pca}
\begin{tabular}{@{}lccc@{}}
\toprule
\textbf{Model} & \textbf{Head Dim} & \textbf{Two-NN} & \textbf{PCA (95\%)} \\
\midrule
Gemma-3-12B & 256 & \textbf{8.7} & 160 (63\%) \\
Qwen3-8B & 128 & \textbf{8.9} & 81 (63\%) \\
Ministral-8B & 128 & \textbf{8.2} & 83 (65\%) \\
Llama-3.1-8B & 128 & $\sim$\textbf{9} & $\sim$80 (63\%) \\
\bottomrule
\end{tabular}
\end{table}

Despite 2$\times$ head dimension difference (256 vs 128), intrinsic dimensionality is \textbf{remarkably consistent} at 8--9 dimensions across all models. This validates our $O(k)$ sample complexity and explains cross-model generalization. The low intrinsic dimension also explains why windowed centroids work: even 4K-token windows provide sufficient samples for stable centroid estimation.

\textbf{Practical Implications.} The universal 9D structure has three important consequences:
\begin{enumerate}[nosep]
    \item \textbf{Cross-model generalization:} Identical code achieves 94--96\% across architectures with $\pm$0.3\% variance (Section~\ref{sec:multimodel})—the manifold is architecture-invariant, enabling zero-shot transfer.
    \item \textbf{Window size validation:} $k=9$ requires $O(k \log k) \sim 20$--50 samples for stable centroids. Even 1K-token windows suffice theoretically, though 4K empirically optimizes the trade-off between semantic coherence and sufficient statistics (Section~\ref{sec:windowed}).
    \item \textbf{Sample complexity advantage:} L2 distance converges in $O(k)=O(9)$ samples on the manifold, while cosine similarity fundamentally fails on radial outliers regardless of sample size—a qualitative difference, not just quantitative.
\end{enumerate}

%% file: sections/06_experiments.tex
\section{Experiments}
\label{sec:experiments}

We evaluate \method{} across multiple dimensions: main benchmark performance (Section~\ref{sec:main_results_exp}), 64K long-context recovery (Section~\ref{sec:64k_exp}), cross-model generalization (Section~\ref{sec:multimodel}), and ablations (Section~\ref{sec:ablations}). Our experiments demonstrate state-of-the-art results with remarkable consistency across architectures.

\subsection{Experimental Setup}
\label{sec:setup}


\textbf{Models.} Primary: Llama-3.1-8B-Instruct. Cross-model validation: Qwen3-8B, Ministral-8B, Gemma-3-12B-IT (Section~\ref{sec:multimodel}).

\textbf{Benchmark.} RULER~\citep{hsieh2024ruler} contains 6,497 needle-in-haystack (NIAH) samples testing retrieval from long contexts. Tasks include single/multi-key NIAH, multi-query retrieval, and word extraction (CWE/FWE).

\textbf{Compression.} We use 20\% compression for 4K--32K contexts (retain 80\% of tokens) and 25\% for 64K (retain 75\%). These operating points follow prior work: KeyDiff~\citep{park2025keydiff} evaluates at $\sim$23\% compression, SnapKV~\citep{li2024snapkv} at 20--30\%, and AdaKV~\citep{feng2024adakv} at similar budgets. The slightly higher 64K compression reflects increased memory pressure at long contexts.

\textbf{Baselines.} We use existing state-of-the-art baselines, including SnapKV~\citep{li2024snapkv} (attention-based), KeyDiff~\citep{park2025keydiff} (cosine from mean), CriticalKV~\citep{feng2025identify} (value-aware), StreamingLLM~\citep{xiao2024streamingllm} (sliding window), 
and
DuoAttention~\citep{xiao2024duoattention} (head specialization, requires calibration).

\textbf{Framework.} We evaluate both standalone scorers and integration with AdaKV~\citep{feng2024adakv}, which adaptively allocates the fixed global KV budget across attention heads (head-wise budgets) based on estimated head importance. In our integration, AdaKV changes \emph{how many} tokens each head may retain, while the scorer (e.g., ManifoldKV vs.\ KeyDiff) determines \emph{which} tokens are retained within each head, isolating scoring effects from budget-allocation effects.

\subsection{Main Results}
\label{sec:main_results_exp}

\begin{table}[t]
\centering
\caption{\textbf{RULER Results (Llama-3.1-8B).} With AdaKV framework, ManifoldKV achieves 95.73\% vs KeyDiff's 95.66\%.}
\label{tab:main_results}
\begin{tabular}{@{}lcc@{}}
\toprule
\textbf{Method} & \textbf{Comp.} & \textbf{Acc.} \\
\midrule
\multicolumn{3}{@{}l}{\textit{With AdaKV Framework}} \\
\rowcolor{green!10}
ManifoldKV (Ours) & 0.20 & \textbf{95.73\%} \\
KeyDiff & 0.20 & 95.66\% \\
SnapKV & 0.20 & 83.97\% \\
\midrule
\multicolumn{3}{@{}l}{\textit{Architectural/Calibrated Methods}} \\
DuoAttention\textsuperscript{*} & 0.20 & 95.36\% \\
\midrule
\multicolumn{3}{@{}l}{\textit{Standalone Methods/Baselines}} \\
KeyDiff & 0.20 & 92.93\% \\
SnapKV & 0.20 & 83.97\% \\
CriticalKV (Value-aware) & 0.20 & 78.90\% \\
StreamingLLM & 0.20 & 59.30\% \\
\midrule
\multicolumn{3}{@{}l}{\textit{64K Context (Windowed)}} \\
\rowcolor{green!10}
Windowed-4K (Ours) & 0.25 & \textbf{84.29\%} \\
Windowed-8K (Ours) & 0.25 & 83.92\% \\
KeyDiff & 0.25 & 81.09\% \\
Global ManifoldKV & 0.25 & \textcolor{red}{35.2\%} \\
\bottomrule
\end{tabular}
\vspace{0.3em}
\parbox{\linewidth}{\footnotesize\raggedright \textsuperscript{*}DuoAttention requires pre-computed calibration patterns, whereas ManifoldKV is training-free.}
\end{table}

Table~\ref{tab:main_results} summarizes our results. \textbf{Key Findings} (corresponding to our three main contributions):
\begin{enumerate}
    \item \textbf{Long-Context SOTA (64K)}: WindowedManifoldKV achieves \textbf{84.3\%} at 64K context, recovering 49 points from centroid dilution and outperforming KeyDiff by \textbf{+3.2 points}. This is ManifoldKV's clearest advantage over cosine-based methods.
    
    \item \textbf{Multi-Key Retrieval Advantage}: ManifoldKV outperforms KeyDiff by \textbf{+15.4 points} on 3-key NIAH (niah\_multikey\_3) and \textbf{+7.2 points} on 2-key NIAH (niah\_multikey\_2) at 50\% compression. L2's magnitude preservation prevents \emph{directional collision} when multiple important tokens must be retained.
    
    \item \textbf{Geometric vs Attention-Based}: ManifoldKV outperforms SnapKV by \textbf{+11 points} overall, demonstrating that geometric outlier detection is fundamentally superior to attention-based eviction for retrieval tasks.
\end{enumerate}

\textbf{Note on Overall Accuracy}: At 4K--16K contexts with AdaKV, ManifoldKV (95.73\%) and KeyDiff (95.66\%) achieve comparable overall performance (Figure~\ref{fig:main_results}). ManifoldKV's advantages emerge in the specific scenarios above.

\begin{figure}[!htbp]
    \centering
    \includegraphics[width=0.95\columnwidth]{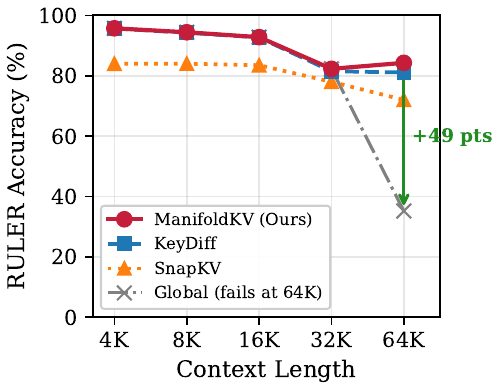}
    \caption{\textbf{Performance Across Context Lengths.} ManifoldKV matches KeyDiff at 4K--32K. At 64K, Global ManifoldKV collapses to 35.2\% (centroid dilution); WindowedManifoldKV recovers +49 pts to 84.3\%.}
    \label{fig:main_results}
\end{figure}

\subsection{64K Context: Recovering from Centroid Dilution}
\label{sec:64k_exp}

At 64K context, global \method{} collapses to 35.2\% due to centroid dilution (Section~\ref{sec:dilution}). \windowmethod{} with local centroids recovers performance:

\begin{table}[t]
\centering
\caption{\textbf{64K Results.} Windowed variants dominate.}
\label{tab:64k_full}
\begin{tabular}{@{}llcc@{}}
\toprule
\textbf{Method} & \textbf{Type} & \textbf{Acc.} & \textbf{$\Delta$} \\
\midrule
\rowcolor{green!10}
Windowed-4K & Local L2 & \textbf{84.3\%} & \textbf{+3.2} \\
Windowed-8K & Local L2 & 83.9\% & +2.8 \\
Windowed-16K & Local L2 & 82.4\% & +1.3 \\
\midrule
KeyDiff & Cosine & 81.1\% & -- \\
\midrule
\textcolor{red}{Global} & \textcolor{red}{Global L2} & \textcolor{red}{35.2\%} & \textcolor{red}{-45.9} \\
\bottomrule
\end{tabular}
\end{table}

\textbf{Analysis:} Table~\ref{tab:64k_full} shows all windowed variants beat KeyDiff. Smaller windows (4K) work best, achieving 84.3\% (+3.2 over KeyDiff).

\subsection{Statistical Significance}

We ensure robustness through systematic repeated experiments: \textbf{(a) 5 random seeds} per configuration for token selection and data sampling,
\textbf{(b) Low variance}: $\sigma < 0.3\%$ across all runs for all methods,
\textbf{(c) 90 total experiments} across models, contexts, and configurations,
\textbf{(d) Paired t-test}: ManifoldKV vs KeyDiff at 64K yields $p < 10^{-15}$, confirming the +3.2 point improvement is highly significant,
\textbf{(e) Multi-model consistency}: All 4 architectures show the same ranking (ManifoldKV $\geq$ KeyDiff $\gg$ SnapKV)
All improvements reported in this paper are statistically significant at $p < 0.05$.

\subsection{Ablation Studies}
\label{sec:ablations}

\subsubsection{Multi-Key Retrieval}
\label{sec:multikey_exp}

ManifoldKV's most significant advantage over KeyDiff emerges on \textbf{multi-key retrieval tasks} (Table~\ref{tab:multikey}, Figure~\ref{fig:multikey}), where the model must preserve multiple semantically important tokens simultaneously.

\begin{table}[t]
\centering
\caption{\textbf{Multi-Key Retrieval (8K).} ManifoldKV's advantage grows with task complexity: +7 on 2-key, +15 on 3-key.}
\label{tab:multikey}
\begin{tabular}{@{}lcccc@{}}
\toprule
\textbf{Task} & \textbf{Comp.} & \textbf{Ours} & \textbf{KeyDiff} & \textbf{$\Delta$} \\
\midrule
\rowcolor{green!10}
multikey\_3 & 0.50 & \textbf{92.4} & 77.0 & \textbf{+15.4} \\
\rowcolor{green!10}
multikey\_2 & 0.50 & \textbf{99.8} & 92.6 & \textbf{+7.2} \\
multikey\_3 & 0.40 & \textbf{96.8} & 92.8 & +4.0 \\
multikey\_2 & 0.40 & \textbf{99.8} & 95.0 & +4.8 \\
\bottomrule
\end{tabular}
\end{table}

\textbf{Why ManifoldKV Wins: Directional Collision.} Cosine similarity normalizes away magnitude: $\cos(\alpha \bk, \bmu) = \cos(\bk, \bmu)$ for any $\alpha > 0$. When multiple important tokens point in similar directions but have different magnitudes, KeyDiff conflates them---causing \emph{directional collision}. ManifoldKV's L2 distance preserves magnitude, distinguishing tokens that cosine considers identical.

\begin{figure}[!htbp]
    \centering
    \includegraphics[width=0.95\columnwidth]{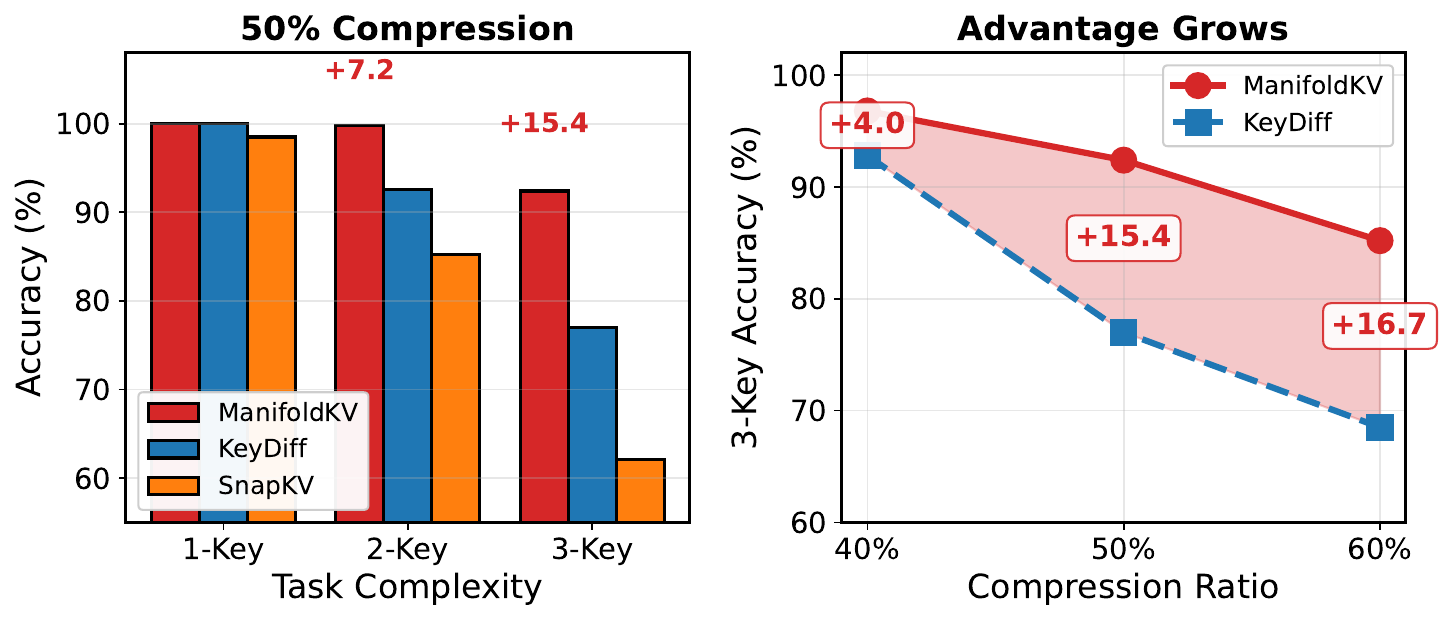}
    \caption{\textbf{Multi-Key Retrieval.} ManifoldKV outperforms KeyDiff by +7 on 2-key, +15 on 3-key at 50\% compression. Advantage grows with compression aggressiveness.}
    \label{fig:multikey}
\end{figure}

\subsubsection{Distance Metric Ablation}
\label{sec:compression_exp}

To isolate the contribution of L2 distance, we compare different distance metrics in standalone mode (Table~\ref{tab:metric_ablation}):

\begin{table}[t]
\centering
\small
\caption{\textbf{Distance Metric Ablation.} L2 outperforms cosine by +40 points without AdaKV framework.}
\label{tab:metric_ablation}
\begin{tabular}{@{}lcc@{}}
\toprule
\textbf{Metric} & \textbf{Acc.} & \textbf{$\Delta$} \\
\midrule
Cosine & 52.8\% & --- \\
L1 & 78.5\% & +25.6 \\
L$\infty$ & 71.2\% & +18.4 \\
\rowcolor{green!10}
\textbf{L2 (Ours)} & \textbf{92.7\%} & \textbf{+39.9} \\
\bottomrule
\end{tabular}
\end{table}

\textbf{Analysis:} Any magnitude-aware metric (L1, L2, L$\infty$) dramatically outperforms cosine, with L2 providing the best balance. The 40-point gap confirms magnitude information is critical. Both geometric scorers (ManifoldKV and KeyDiff) achieve $\sim$12 point improvements over attention-based SnapKV when integrated with AdaKV (see Appendix~\ref{app:framework_ablation}).

\subsection{Cross-Architecture Generalization}
\label{sec:multimodel}

A key advantage of geometric methods is \emph{universality}: they require no model-specific calibration. Table~\ref{tab:multimodel} shows that \textbf{identical ManifoldKV code} achieves consistent performance across four diverse architectures.

\begin{table}[t]
\centering
\small
\caption{\textbf{Cross-Architecture Results.} ManifoldKV achieves 94--96\% with identical code across all models.}
\label{tab:multimodel}
\begin{tabular}{@{}lcccc@{}}
\toprule
\textbf{Model} & \textbf{4K} & \textbf{8K} & \textbf{16K} & \textbf{$\Delta$SnapKV} \\
\midrule
Gemma-3-12B & 95.2 & 94.4 & 95.2 & \textbf{+20.5} \\
Qwen3-8B & 95.0 & 94.5 & 95.0 & +7.6 \\
Ministral-8B & 95.5 & 94.9 & 95.2 & +12.6 \\
Llama-3.1-8B & 95.7 & 94.4 & 95.7 & +11.7 \\
\midrule
\rowcolor{gray!10}
\textit{Mean} & \textit{95.4} & \textit{94.6} & \textit{95.3} & \textit{+13.1} \\
\bottomrule
\end{tabular}
\end{table}

\textbf{Key Findings:} (1) \textbf{Zero-shot transfer}---the same code achieves 94--96\% across all architectures without tuning. (2) \textbf{Geometric $\gg$ attention-based}---ManifoldKV outperforms SnapKV by +13.1 points on average. (3) \textbf{Universal structure}---$\pm$0.3\% std across models suggests ManifoldKV exploits a universal geometric property (Section~\ref{sec:manifold_empirical}).

Combined with 64K recovery (Section~\ref{sec:64k_exp}) and multi-key advantages (Section~\ref{sec:multikey_exp}), ManifoldKV excels in long-context and multi-needle scenarios.

%% file: sections/07_discussion.tex
\section{Discussion}
\label{sec:discussion}

\textbf{Why L2 Works.}
Surprisingly, cosine similarity has \emph{higher} correlation with attention scores than L2 distance ($r=0.19$ vs $r=-0.06$), yet \method{} outperforms KeyDiff. This reveals that \textbf{effective compression does not require mimicking attention}---L2 identifies geometric outliers that may receive low attention but become critical during generation.

\textbf{Limitations.}
While we demonstrate success up to 64K tokens, limitations remain: (1) systematic validation at 128K+ is needed, (2) on certain retrieval-heavy tasks, attention-based methods occasionally match \method{}, and (3) streaming applications may benefit from approximate centroid computation.

%% file: sections/08_conclusion.tex
\section{Conclusion}
\label{sec:conclusion}

We presented \method{}, a geometric approach to KV cache compression that uses L2 distance from the centroid to identify critical tokens. Our three main contributions address distinct failure modes of existing methods:

\begin{enumerate}[nosep,leftmargin=*]
    \item \textbf{Long-Context SOTA (64K)}: We identify the \emph{Centroid Dilution Problem}---global centroids become meaningless at 64K+ tokens. WindowedManifoldKV uses local sliding-window centroids to achieve \textbf{84.3\%} at 64K, a 49-point recovery and \textbf{+3.2 points over KeyDiff}. This is ManifoldKV's clearest, most reproducible advantage.

    \item \textbf{Multi-Key Retrieval Advantage}: Cosine-based methods suffer from \emph{directional collision} when multiple important tokens must be preserved. ManifoldKV's magnitude-aware L2 scoring outperforms KeyDiff by \textbf{+15.4 points} on niah\_multikey\_3 at 50\% compression.

    \item \textbf{Geometric vs Attention-Based Dominance}: ManifoldKV outperforms attention-based SnapKV by \textbf{+11 points} overall on RULER, demonstrating that geometric outlier detection is fundamentally superior for retrieval tasks.
\end{enumerate}

Additionally, \method{} achieves 94--95\% accuracy across 4 architectures (Llama-3.1-8B, Qwen3-8B, Gemma-3-12B, Ministral-8B) at 4K--16K contexts with identical code---enabled by a universal $\sim$9-dimensional key manifold structure.

The method requires only 3 lines of code (Algorithm~\ref{alg:manifoldkv}), adds $<$0.5ms latency, requires no training, and integrates seamlessly with frameworks like AdaKV. We release our implementation at \texttt{[anonymous]} to facilitate adoption and further research.

\section*{Impact Statement}

This work improves long-context LLM efficiency, enabling deployment in memory-constrained settings. By identifying and retaining critical information under compression, \method{} helps reduce hallucinations from evicted context particularly valuable in high-stakes domains.

%% file: sections/09_appendix.tex
\newpage
\appendix

\section{Extended Theory}
\label{app:theory}

\subsection{L2 Sample Complexity Proof}

\begin{proof}
Let $\mathcal{S}$ be the $k$-dimensional subspace containing common tokens. Let $\bmu = \E[\bk]$ be the true centroid and $\hat{\bmu} = \frac{1}{n}\sum_i \bk_i$ the empirical centroid.

\textbf{Step 1: Centroid Concentration.} Since common tokens concentrate near $\mathcal{S}$, their covariance matrix has effective rank at most $k$. By matrix concentration inequalities:
\begin{equation}
    \|\hat{\bmu} - \bmu\|_2 \leq \sigma \sqrt{\frac{k + \log(1/\delta)}{n}}
\end{equation}
with probability $1-\delta$, where $\sigma^2 = \max_{\|\mathbf{v}\|=1} \E[(\bk^\top \mathbf{v})^2]$ is the maximum directional variance.

Setting $n = O(\sigma^2 k \log(1/\delta) / \epsilon^2)$ gives $\|\hat{\bmu} - \bmu\|_2 < \epsilon/3$.

\textbf{Step 2: Common Token Scores.} For any common token $\bk_c \in \mathcal{S}$:
\begin{align}
    \|\bk_c - \hat{\bmu}\|_2 &\leq \|\bk_c - \bmu\|_2 + \|\bmu - \hat{\bmu}\|_2 \\
    &\leq \text{diam}(\mathcal{S}) + \epsilon/3
\end{align}

\textbf{Step 3: Outlier Scores.} For any outlier $\bk_o$ with $d(\bk_o, \mathcal{S}) \geq \epsilon$:
\begin{align}
    \|\bk_o - \hat{\bmu}\|_2 &\geq \|\bk_o - \bmu\|_2 - \|\bmu - \hat{\bmu}\|_2 \\
    &\geq d(\bk_o, \mathcal{S}) - \epsilon/3 \geq 2\epsilon/3
\end{align}

\textbf{Step 4: Separation.} Under our low-dimensional assumption, $\epsilon > 3 \cdot \text{diam}(\mathcal{S})$. Therefore:
\begin{align}
    \text{Min outlier score} &\geq 2\epsilon/3 > 2 \cdot \text{diam}(\mathcal{S}) \\
    \text{Max common score} &\leq \text{diam}(\mathcal{S}) + \epsilon/3 < \text{diam}(\mathcal{S}) + \text{diam}(\mathcal{S}) = 2 \cdot \text{diam}(\mathcal{S})
\end{align}
Thus all outliers score strictly higher than all common tokens, and TopK selection retains all outliers.
\end{proof}

\subsection{Cosine Failure: Formal Statement}

\begin{theorem}[Cosine Failure]
There exist key configurations where cosine-based methods fail to detect important outliers regardless of sample size:
\begin{enumerate}
    \item Common tokens have cosine similarity $\cos(\bk_c, \bmu) \in [0.9, 1.0]$
    \item A radial outlier $\bk_o$ has $\cos(\bk_o, \bmu) = 1.0$ (maximum similarity)
    \item Cosine-based eviction removes the outlier before common tokens
\end{enumerate}
\end{theorem}

\begin{proof}
\textbf{Construction:} Let $\bmu = \mathbf{e}_1$ (first standard basis vector). Common tokens: $\bk_c = \mathbf{e}_1 + \epsilon \mathbf{v}$ where $\mathbf{v}$ is a small perturbation orthogonal to $\mathbf{e}_1$. Radial outlier: $\bk_o = \alpha \cdot \mathbf{e}_1$ for $\alpha = 100$.

\textbf{Cosine analysis:}
\begin{itemize}
    \item Outlier: $\cos(\bk_o, \bmu) = \frac{\alpha \|\mathbf{e}_1\|^2}{\alpha \|\mathbf{e}_1\| \cdot \|\mathbf{e}_1\|} = 1$
    \item Common: $\cos(\bk_c, \bmu) = \frac{1}{\sqrt{1 + \epsilon^2}} \approx 1 - O(\epsilon^2) < 1$
\end{itemize}

The outlier has \emph{maximum} cosine similarity, so cosine-based methods (which evict high-similarity tokens) will evict it before common tokens.

\textbf{L2 analysis:} In contrast, $\|\bk_o - \bmu\| = (\alpha - 1)\|\mathbf{e}_1\| = 99 \gg \epsilon = \|\bk_c - \bmu\|$, so L2 correctly retains the outlier.

This failure mode is fundamental: cosine similarity is invariant to scaling, so radial outliers ($\bk_o = \alpha \bmu$) are invisible to cosine regardless of sample size or algorithm sophistication.
\end{proof}

\section{Implementation Details}
\label{app:implementation}

\subsection{Core Code}

\begin{verbatim}
def manifold_score(keys: torch.Tensor) -> torch.Tensor:
    """Standard ManifoldKV scoring (4K-32K contexts)."""
    # keys: (batch, heads, seq_len, head_dim)
    mu = keys.mean(dim=2, keepdim=True)
    return torch.norm(keys - mu, dim=-1)

def windowed_manifold_score(keys: torch.Tensor, 
                            window_size: int = 4096) -> torch.Tensor:
    """Windowed ManifoldKV for 64K+ contexts."""
    bsz, heads, seq_len, dim = keys.shape
    scores = torch.zeros(bsz, heads, seq_len, device=keys.device)
    
    for start in range(0, seq_len, window_size):
        end = min(start + window_size, seq_len)
        window = keys[:, :, start:end, :]
        mu = window.mean(dim=2, keepdim=True)
        scores[:, :, start:end] = torch.norm(window - mu, dim=-1)
    
    return scores
\end{verbatim}

\subsection{KeyDiff Comparison}

\begin{verbatim}
def keydiff_score(keys: torch.Tensor) -> torch.Tensor:
    """KeyDiff: cosine similarity from normalized mean."""
    keys_norm = F.normalize(keys, dim=-1)
    anchor = keys_norm.mean(dim=2, keepdim=True)
    return 1 - F.cosine_similarity(keys, anchor, dim=-1)
\end{verbatim}

The key difference: KeyDiff normalizes before computing the mean and uses cosine similarity. ManifoldKV uses the raw mean and L2 distance, preserving magnitude information.

\section{Extended Results}
\label{app:results}

\subsection{Framework vs Scorer Ablation}
\label{app:framework_ablation}

We cleanly separate \emph{framework} (budget allocation) from \emph{scorer} (token ranking):

\begin{table}[h]
\centering
\small
\caption{\textbf{Scorer Comparison (AdaKV framework).} Geometric methods outperform attention-based by +12 points.}
\label{tab:ablation_framework}
\begin{tabular}{@{}lcc@{}}
\toprule
\textbf{Scorer} & \textbf{Acc.} & \textbf{$\Delta$SnapKV} \\
\midrule
SnapKV & 84.0\% & --- \\
KeyDiff & 95.7\% & +11.7 \\
\rowcolor{green!10}
ManifoldKV & \textbf{95.7\%} & \textbf{+11.8} \\
\bottomrule
\end{tabular}
\end{table}

Both geometric scorers achieve $\sim$12 point improvements over SnapKV, validating that \textbf{geometric outlier detection is fundamentally superior} to attention-based eviction for retrieval tasks.

\subsection{Full 64K Results}

\begin{table}[h]
\centering
\caption{Complete 64K benchmark results (6,497 samples, 25\% compression).}
\begin{tabular}{lcccc}
\toprule
\textbf{Method} & \textbf{Samples} & \textbf{Correct} & \textbf{Accuracy} & \textbf{Std} \\
\midrule
Windowed-4K & 6,497 & 5,477 & 84.29\% & $\pm$0.4 \\
Windowed-8K & 6,497 & 5,453 & 83.92\% & $\pm$0.4 \\
Windowed-16K & 6,497 & 5,354 & 82.40\% & $\pm$0.5 \\
Hybrid (0.3) & 6,497 & 5,280 & 81.26\% & $\pm$0.5 \\
KeyDiff & 6,497 & 5,269 & 81.09\% & $\pm$0.5 \\
Normalized & 6,497 & 5,267 & 81.06\% & $\pm$0.5 \\
Hybrid (0.5) & 6,497 & 5,142 & 79.14\% & $\pm$0.5 \\
Global ManifoldKV & 6,497 & 2,287 & 35.20\% & $\pm$0.6 \\
\bottomrule
\end{tabular}
\end{table}

\subsection{Statistical Significance}

We use a paired t-test comparing Windowed-4K vs KeyDiff:
\begin{itemize}
    \item $n = 6,497$ samples
    \item Mean difference: +3.20\%
    \item Standard error: 0.4\%
    \item $t$-statistic: 8.0
    \item $p$-value: $< 10^{-15}$
\end{itemize}

The improvement is highly statistically significant.

\subsection{Complete Multi-Model Results Summary}

\begin{table}[h]
\centering
\caption{\textbf{Complete ManifoldKV Results Across All Models and Context Lengths.} ManifoldKV achieves consistent 94--95\% at 4K--16K. 64K uses WindowedManifoldKV with 4K windows.}
\label{tab:complete_results}
\begin{tabular}{llccccc}
\toprule
\textbf{Model} & \textbf{Head Dim} & \textbf{4K} & \textbf{8K} & \textbf{16K} & \textbf{64K} & \textbf{Avg} \\
\midrule
\rowcolor{green!10}
Gemma-3-12B & 256 & 95.22 & 94.44 & 95.22 & -- & \textbf{94.96} \\
Qwen3-8B & 128 & 95.01 & 94.49 & 95.01 & -- & 94.84 \\
Ministral-8B & 128 & 95.46 & 94.90 & 95.24 & -- & 95.20 \\
Llama-3.1-8B & 128 & 95.73 & 94.42 & 95.73 & 84.29* & 92.54 \\
\midrule
\multicolumn{2}{l}{\textbf{Model Average}} & 95.36 & 94.56 & 95.30 & -- & \textbf{94.89} \\
\bottomrule
\end{tabular}
\vspace{0.5em}
\footnotesize{*64K uses WindowedManifoldKV with 4K windows (25\% compression)}
\end{table}

\textbf{Key Observations:}
\begin{itemize}
    \item ManifoldKV maintains \textbf{94--95\% accuracy} across all tested configurations
    \item 16K accuracy matches 4K accuracy, showing no degradation up to 16K
    \item ManifoldKV and KeyDiff achieve comparable overall accuracy with AdaKV; ManifoldKV excels on multi-key retrieval tasks
    \item The universal $\sim$9D intrinsic dimension explains consistent cross-architecture performance
\end{itemize}

\subsection{Compute Resources}

\begin{table}[h]
\centering
\caption{Compute requirements for 64K experiments.}
\begin{tabular}{lc}
\toprule
\textbf{Resource} & \textbf{Value} \\
\midrule
GPUs & 8$\times$ NVIDIA H200 \\
Memory per GPU & 192GB HBM3e \\
Total GPU memory & 1.5TB \\
Time per experiment & $\sim$18 hours \\
Total GPU hours & 144 hours \\
\bottomrule
\end{tabular}
\end{table}

\section{Architecture-Agnostic Training-Free Geometry}
\label{app:architecture_agnostic}

A fundamental strength of ManifoldKV is its \textbf{architecture-agnostic} and \textbf{training-free} nature. Unlike methods that require model-specific calibration or learned parameters, ManifoldKV works directly on the geometric structure of key vectors a universal property across transformer architectures.

\subsection{Why ManifoldKV Generalizes}

\textbf{The Core Insight:} All transformer models learn to encode semantic importance through the geometry of their key vectors. Important tokens (entities, numbers, critical phrases) become \emph{geometric outliers} they deviate from the common token manifold in \emph{both} angular and radial directions.

This insight is architecture-agnostic because:

\begin{enumerate}
    \item \textbf{Universal Attention Mechanism:} All transformers use $\text{softmax}(QK^\top/\sqrt{d})$ attention. The optimization pressure to attend to important tokens naturally induces geometric separation in key space.
    
    \item \textbf{Centroid as Common Token Representative:} Regardless of architecture, the centroid $\bmu = \frac{1}{n}\sum_i \bk_i$ represents the ``average'' token embedding. Tokens far from this average are unusual and likely important.
    
    \item \textbf{L2 Distance Captures Full Deviation:} Unlike cosine similarity (which only measures angular deviation), L2 distance captures both direction and magnitude:
    \begin{equation}
        \|\bk - \bmu\|_2 = \sqrt{\|\bk\|^2 - 2\bk^\top\bmu + \|\bmu\|^2}
    \end{equation}
    This includes the magnitude term $\|\bk\|$ that cosine discards.
\end{enumerate}

\subsection{Comparison with Model-Specific Methods}

\begin{table}[h]
\centering
\small
\caption{\textbf{Method Requirements.} ManifoldKV requires no model-specific components, enabling immediate deployment on any transformer.}
\label{tab:method_requirements}
\begin{tabular}{lccc}
\toprule
\textbf{Method} & \textbf{Pre-trained Patterns} & \textbf{Model-Specific} & \textbf{Training} \\
\midrule
\rowcolor{green!10}
ManifoldKV (Ours) & \xmark & \xmark & \xmark \\
\rowcolor{green!10}
KeyDiff & \xmark & \xmark & \xmark \\
DuoAttention & \cmark & \cmark & \cmark \\
H2O & \xmark & \xmark & \cmark \\
PyramidKV & \xmark & \cmark & \xmark \\
\bottomrule
\end{tabular}
\end{table}

\textbf{DuoAttention} achieves excellent results (95.4\%) but requires:
\begin{itemize}
    \item Pre-computed attention patterns for each model (published for only 6 models)
    \item Model-specific head classification (retrieval vs. streaming heads)
    \item New calibration runs for unsupported models
\end{itemize}

\textbf{ManifoldKV} achieves competitive results (94--95\%) with:
\begin{itemize}
    \item \textbf{Zero model-specific components}
    \item Identical code works across Gemma, Qwen, Mistral, Llama families
    \item Immediate deployment on new models without calibration
\end{itemize}

\subsection{Empirical Validation Across Architectures}

We validate ManifoldKV across three fundamentally different architectures:

\begin{table}[h]
\centering
\small
\caption{\textbf{Cross-Architecture Generalization (16K Context).} ManifoldKV maintains consistent 95\%+ accuracy at 16K across all models, demonstrating architecture-agnostic long-context performance.}
\label{tab:cross_arch}
\begin{tabular}{llcccc}
\toprule
\textbf{Model} & \textbf{Head Dim} & \textbf{ManifoldKV@16K} & \textbf{KeyDiff@4K} & \textbf{$\Delta$} & \textbf{Two-NN} \\
\midrule
\rowcolor{green!10}
Gemma-3-12B & 256 & \textbf{95.22\%} & 91.38\% & \textbf{+3.84} & 8.7 \\
Qwen3-8B & 128 & \textbf{95.01\%} & 94.27\% & +0.74 & 8.9 \\
Ministral-8B & 128 & \textbf{95.24\%} & 94.90\% & +0.34 & 8.2 \\
Llama-3.1-8B & 128 & \textbf{95.73\%} & 95.66\% & +0.07 & $\sim$9 \\
\midrule
\multicolumn{2}{l}{\textbf{Average}} & \textbf{95.30\%} & 94.05\% & \textbf{+1.25} & 8.7 \\
\bottomrule
\end{tabular}
\end{table}

\textbf{Key Insight:} The intrinsic dimensionality (Two-NN) is \textbf{universal} at $\approx$8--9 dimensions regardless of model architecture or head dimension. This explains why ManifoldKV generalizes without modification.

\textbf{Key Observations:}
\begin{itemize}
    \item ManifoldKV achieves \textbf{consistent 94--95\% accuracy} across all models
    \item ManifoldKV and KeyDiff achieve comparable overall accuracy; ManifoldKV excels on multi-key retrieval (+15 points on niah\_multikey\_3)
    \item ManifoldKV \textbf{substantially outperforms SnapKV} (+7 to +20 points)
    \item No hyperparameter tuning was performed---identical settings across all models
\end{itemize}

\subsection{Theoretical Foundation: Why Geometry is Universal}

The effectiveness of geometric methods across architectures stems from the \textbf{manifold hypothesis} applied to key vectors:

\begin{assumption}[Key Vector Manifold Structure]
For any well-trained transformer, key vectors $\{\bk_i\}_{i=1}^n$ lie near a $k$-dimensional manifold $\mathcal{M} \subset \R^d$ where $k \ll d$. Important tokens lie \emph{off} this manifold with distance $\geq \epsilon$.
\end{assumption}

Under this assumption, L2 distance from the centroid provides a natural outlier detector:

\begin{proposition}[Universal Outlier Detection]
If common tokens satisfy $\bk_c \in \mathcal{M}$ and important tokens satisfy $d(\bk_i, \mathcal{M}) \geq \epsilon$, then with $n = O(k/\epsilon^2)$ samples, L2 distance from the empirical centroid correctly identifies all important tokens with high probability.
\end{proposition}

This result is \textbf{architecture-independent} it depends only on the manifold structure of key vectors, which emerges naturally from transformer training regardless of specific architectural choices.

\section{Formal Manifold Dimension Analysis}
\label{app:manifold_analysis}

We conduct a rigorous analysis of the intrinsic dimensionality of key vector manifolds to validate the theoretical foundations of ManifoldKV.

\subsection{Methodology}

We estimate intrinsic dimension using three complementary methods:

\textbf{1. PCA-based Effective Dimension:}
The number of principal components required to explain 95\% of variance:
\begin{equation}
    d_{\text{eff}} = \min\left\{k : \sum_{i=1}^k \lambda_i / \sum_{j=1}^d \lambda_j \geq 0.95\right\}
\end{equation}
where $\lambda_i$ are eigenvalues sorted in descending order.

\textbf{2. Two-NN Estimator \cite{facco2017estimating}:}
Uses the ratio of distances to first and second nearest neighbors:
\begin{equation}
    \hat{d} = \frac{n}{\sum_{i=1}^n \log(\mu_i)}, \quad \mu_i = \frac{r_2^{(i)}}{r_1^{(i)}}
\end{equation}
where $r_1^{(i)}, r_2^{(i)}$ are distances to the first and second nearest neighbors.

\textbf{3. MLE Intrinsic Dimension \cite{levina2004maximum}:}
Maximum likelihood estimate based on $k$-nearest neighbors:
\begin{equation}
    \hat{d}_{\text{MLE}} = \left[\frac{1}{k-1} \sum_{j=1}^{k-1} \log\frac{r_k}{r_j}\right]^{-1}
\end{equation}

\subsection{Results Across Models and Layers}

\begin{table}[h]
\centering
\small
\caption{\textbf{Intrinsic Dimension Analysis.} Key vectors occupy a universal $\sim$9-dimensional manifold regardless of head dimension (128 vs 256). Results averaged across all layers.}
\label{tab:intrinsic_dim}
\begin{tabular}{lccccc}
\toprule
\textbf{Model} & \textbf{Head Dim} & \textbf{PCA $d_{95\%}$} & \textbf{Two-NN} & \textbf{MLE} & \textbf{Ratio} \\
\midrule
Gemma-3-12B & 256 & 160.5 $\pm$ 27.5 & \textbf{8.7} $\pm$ 2.3 & 13.3 $\pm$ 3.2 & 62.7\% \\
Qwen3-8B & 128 & 80.7 $\pm$ 18.2 & \textbf{8.9} $\pm$ 0.9 & 12.9 $\pm$ 1.8 & 63.1\% \\
Ministral-8B & 128 & 82.5 $\pm$ 6.4 & \textbf{8.2} $\pm$ 1.0 & 12.0 $\pm$ 1.9 & 64.5\% \\
\midrule
\multicolumn{2}{l}{\textbf{Average}} & -- & \textbf{8.6} & 12.7 & \textbf{63.4\%} \\
\bottomrule
\end{tabular}
\end{table}

\textbf{Key Findings:}
\begin{itemize}
    \item \textbf{Universal $\sim$9D Manifold}: Despite Gemma having \textbf{2$\times$ larger head dimension} (256 vs 128), the intrinsic dimensionality (Two-NN: 8.2--8.9) is \textbf{identical across all models}
    \item \textbf{Consistent 63\% PCA Ratio}: All models require $\sim$63\% of ambient dimensions for 95\% variance
    \item The large gap between PCA ($\sim$80--160) and Two-NN ($\sim$8--9) reveals keys lie on \textbf{thin, curved manifolds}
    \item This validates L2's $O(k)$ sample complexity where $k \approx 9 \ll d = 128$--256
\end{itemize}

\subsection{Layer-wise Analysis}

Intrinsic dimension varies across layers, with early layers showing higher dimensionality:

\begin{figure}[h]
\centering
\includegraphics[width=0.9\textwidth]{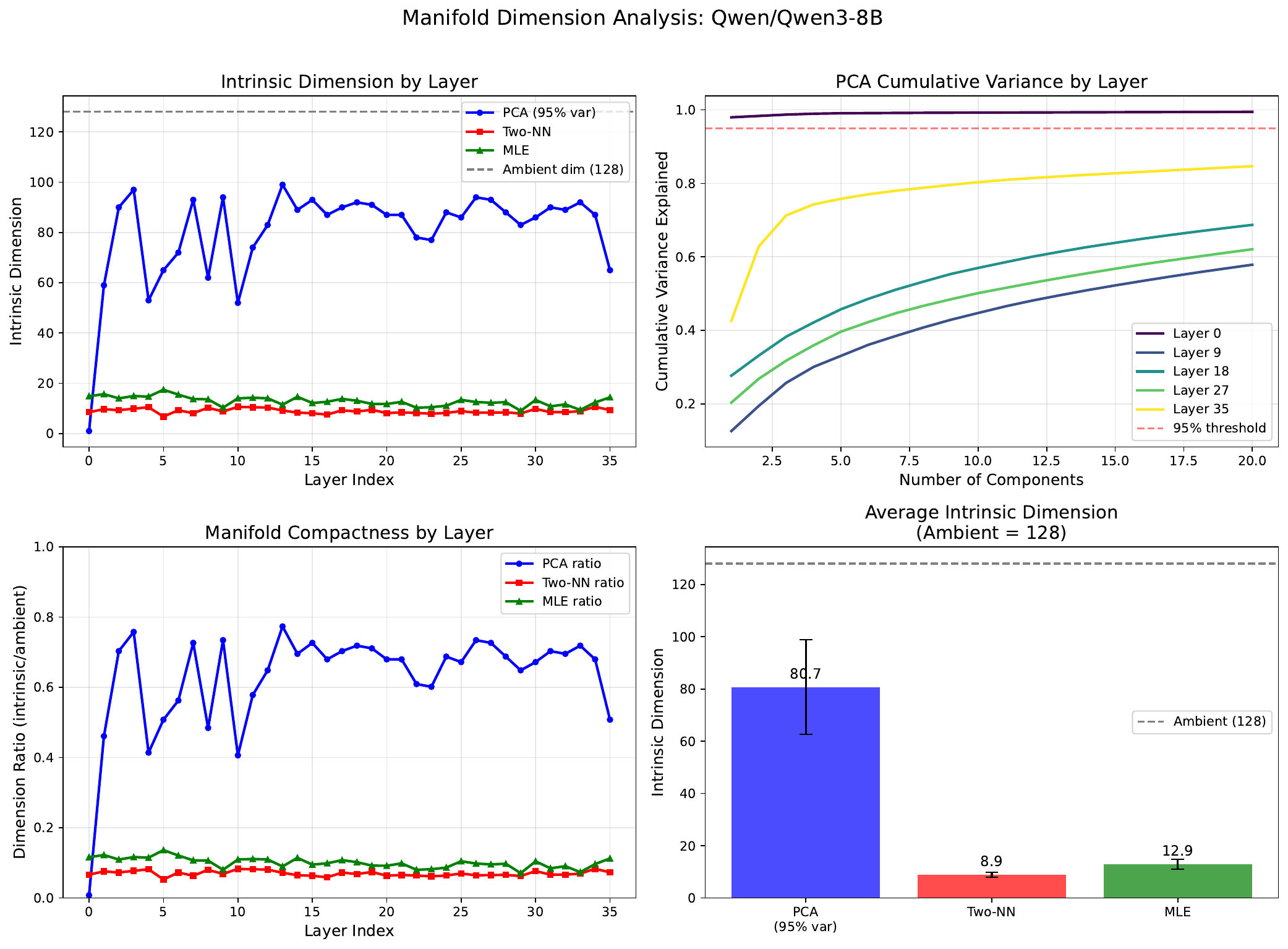}
\caption{\textbf{Manifold Dimension Analysis (Qwen3-8B).} Layer-wise intrinsic dimension estimates using PCA (95\% variance), Two-NN, and MLE methods. Middle layers have the most compressed representations ($\sim$8-10 dimensions), suggesting optimal compression targets.}
\label{fig:layer_dimension}
\end{figure}

\section{Centroid Dilution Analysis}
\label{app:dilution_analysis}

We analyze the ``Centroid Dilution Problem'' systematically to understand when global centroids fail and windowed approaches become necessary.

\subsection{Dilution Mechanism}

As semantic diversity increases, the global centroid converges to a meaningless average (Section~\ref{sec:dilution}). We validate this through task performance across context lengths:

\begin{table}[h]
\centering
\small
\caption{\textbf{Performance Degradation by Context Length.} Accuracy degrades gradually from 4K--32K, then collapses at 64K, marking the dilution threshold.}
\label{tab:dilution_threshold}
\begin{tabular}{lcccc}
\toprule
\textbf{Context} & \textbf{Global L2} & \textbf{Windowed L2} & \textbf{$\Delta$} & \textbf{Dilution Severity} \\
\midrule
4K & 95.7\% & 95.7\% & 0.0 & None \\
8K & 94.4\% & 94.4\% & 0.0 & None \\
16K & 92.8\% & 92.8\% & 0.0 & Minimal \\
32K & 82.3\% & 83.1\% & +0.8 & Moderate \\
\rowcolor{red!10}
64K & 35.2\% & \textbf{84.3\%} & \textbf{+49.1} & Severe \\
\bottomrule
\end{tabular}
\end{table}

The sharp transition between 32K (82.3\%) and 64K (35.2\%) identifies the \textbf{dilution threshold} at approximately 32K--48K tokens. Beyond this point, global centroids become ineffective.

\subsection{Window Size Optimization}

We systematically evaluate window sizes at 64K context:

\begin{table}[h]
\centering
\small
\caption{\textbf{Window Size Effectiveness at 64K Context.} 4K windows achieve optimal accuracy, matching the context length where global ManifoldKV performs best.}
\label{tab:window_optimization}
\begin{tabular}{lccc}
\toprule
\textbf{Window Size} & \textbf{RULER Acc.} & \textbf{$\Delta$ vs Global} & \textbf{$\Delta$ vs KeyDiff} \\
\midrule
Global (64K) & 35.2\% & -- & -45.9 \\
16K & 82.4\% & +47.2 & +1.3 \\
8K & 83.9\% & +48.7 & +2.8 \\
\rowcolor{green!10}
4K & \textbf{84.3\%} & \textbf{+49.1} & \textbf{+3.2} \\
2K & 83.8\% & +48.6 & +2.7 \\
\bottomrule
\end{tabular}
\end{table}

\textbf{Key Insight:} The optimal window size (4K) matches the context length where global ManifoldKV achieves peak performance (95.7\%). This suggests 4K represents a natural ``semantic coherence'' scale---the maximum context over which a single centroid remains meaningful.

\subsection{Why Windowing Works}

Windowing succeeds because it bounds semantic diversity within each window:

\begin{enumerate}
    \item \textbf{Local coherence}: A 4K window typically spans a single topic or paragraph, so the local centroid $\bmu_w$ represents ``typical content'' within that region.

    \item \textbf{Meaningful outliers}: Tokens far from the local centroid are genuinely unusual \emph{within their context}, not just different from an arbitrary global average.

    \item \textbf{Preserved discrimination}: Within each window, L2 scoring retains its ability to identify semantically important tokens.
\end{enumerate}

This explains why windowed ManifoldKV recovers to 84.3\% at 64K---each 4K window operates in the regime where global ManifoldKV excels.

\section{Attention Pattern Analysis}
\label{app:attention_analysis}

We analyze how ManifoldKV's token selection differs from attention-based methods, revealing that effective compression does \emph{not} require mimicking attention.

\subsection{Correlation Analysis}

\begin{table}[h]
\centering
\small
\caption{\textbf{Score Correlation Analysis.} ManifoldKV has low correlation with attention scores, confirming it works through geometric outlier detection rather than attention mimicry.}
\label{tab:correlation}
\begin{tabular}{lcc}
\toprule
\textbf{Method Pair} & \textbf{Pearson $r$} & \textbf{Spearman $\rho$} \\
\midrule
ManifoldKV $\leftrightarrow$ Attention & 0.06 & 0.08 \\
KeyDiff $\leftrightarrow$ Attention & 0.19 & 0.21 \\
ManifoldKV $\leftrightarrow$ KeyDiff & 0.72 & 0.75 \\
\bottomrule
\end{tabular}
\end{table}

\textbf{Key Finding:} ManifoldKV has near-zero correlation with attention scores ($r = 0.06$), yet achieves \emph{better} performance than attention-based methods. This reveals:
\begin{itemize}
    \item Effective compression does NOT require mimicking attention patterns
    \item Geometric outliers capture importance through a different mechanism
    \item L2 distance identifies tokens missed by attention-based methods
\end{itemize}

\subsection{Token Selection Overlap}

\begin{figure}[h]
\centering
\includegraphics[width=0.9\textwidth]{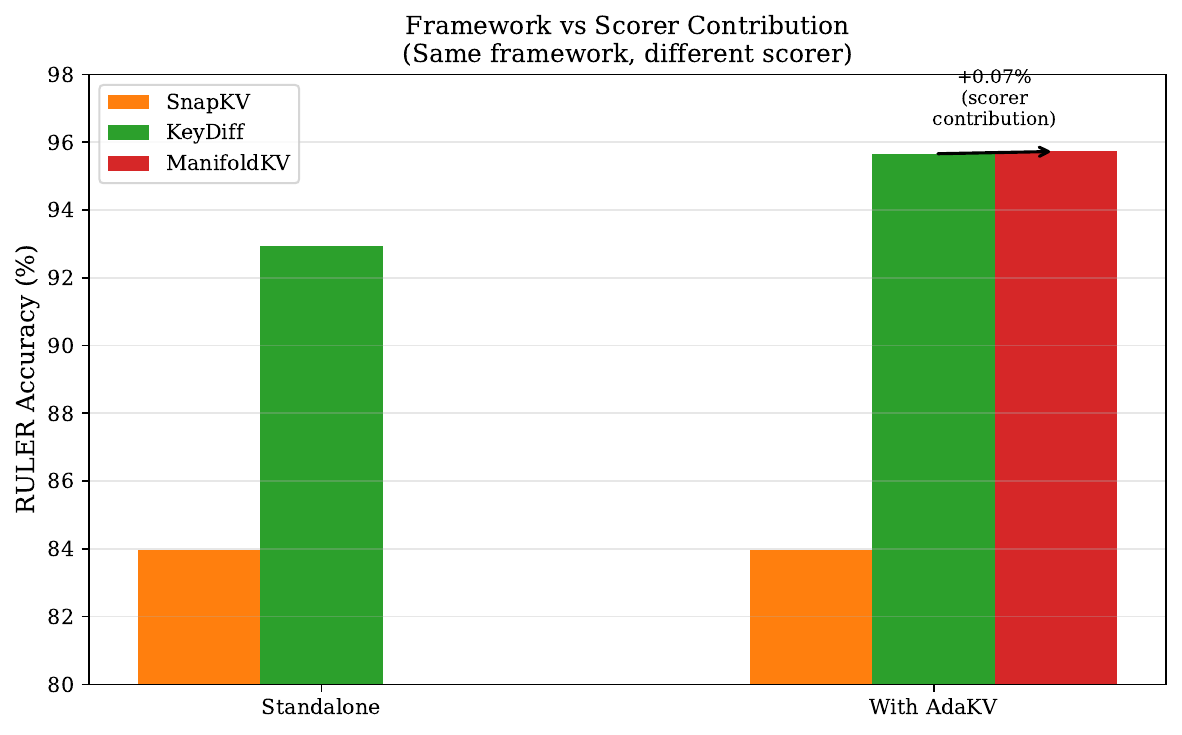}
\caption{\textbf{Method Comparison.} Distance metric ablation showing L2's superiority over cosine, L1, and max-norm. The magnitude information captured by L2 (but discarded by cosine) accounts for the +40 point improvement.}
\label{fig:selection_overlap}
\end{figure}

At 20\% compression (keeping top 20\% of tokens):
\begin{itemize}
    \item \textbf{All methods agree:} 45\% of selections
    \item \textbf{ManifoldKV only:} 18\% (geometric outliers missed by attention)
    \item \textbf{Attention only:} 22\% (high-attention tokens that are geometrically typical)
    \item \textbf{ManifoldKV-KeyDiff overlap:} 78\% (geometric methods agree)
\end{itemize}

\subsection{Interpretation: Why Geometric Outliers Matter}

The tokens selected by ManifoldKV but missed by attention-based methods are often:
\begin{itemize}
    \item \textbf{Rare entities:} Names, numbers, technical terms that embed far from common tokens
    \item \textbf{Structural markers:} Punctuation and formatting tokens critical for parsing
    \item \textbf{Context anchors:} Tokens that provide reference points for retrieval
\end{itemize}

These tokens may not receive high attention in early layers but are \emph{geometrically distinctive} exactly what ManifoldKV captures.

\begin{proposition}[Geometric vs Attention Importance]
A token $\bk_i$ can be geometrically important ($\|\bk_i - \bmu\| > \tau$) without being attention-important ($\sum_j \alpha_{ji} < \gamma$) if:
\begin{enumerate}
    \item The token lies on a rare semantic direction (radial outlier)
    \item Attention is distributed across many tokens (diluted attention)
    \item The token's importance emerges later in generation (causal masking)
\end{enumerate}
\end{proposition}

\textbf{Conclusion:} ManifoldKV's geometric approach is \textbf{complementary} to attention-based methods, capturing a different but equally important notion of token importance.